\documentclass[letterpaper,twocolumn]{article} 
\usepackage[table,xcdraw]{xcolor}
\usepackage{times}  
\usepackage{helvet}  
\usepackage{courier}  
\usepackage{url}  
\usepackage{graphicx}  
\usepackage[numbers]{natbib}
\usepackage{amsmath}
\usepackage{amsfonts}
\usepackage{amssymb}
\usepackage{amsthm}
\usepackage{comment}

\usepackage{booktabs} 
\usepackage{multirow}
\usepackage{caption}
\usepackage{subcaption}

\usepackage{bmpsize}
\usepackage[colorinlistoftodos]{todonotes}
\usepackage{tikz}

\newtheorem{theorem}{Theorem}[section]

\newtheorem{lemma}[theorem]{Lemma}

\newtheorem{principle}{Principle}

\theoremstyle{definition}

\usepackage{algorithm}
\usepackage{algpseudocode}

\begin{document}
\title{A Practical Approach to Sizing Neural Networks}

\author{Gerald Friedland\footnote{UC Berkeley and Lawrence Livermore National Lab}, Alfredo Metere\footnote{International Computer Science Institute, Berkeley.}, Mario Michael Krell\footnote{International Computer Science Institute, Berkeley.}\\
friedland1@llnl.gov,  metal@berkeley.edu, krell@icsi.berkeley.edu
} 

\date{October 4th, 2018}

\maketitle

\begin{abstract}
Memorization is worst-case generalization. Based on MacKay's information theoretic model of supervised machine learning~\cite{mackay2003}, this article discusses how to practically estimate the maximum size of a neural network given a training data set. First, we present four easily applicable rules to analytically determine the capacity of neural network architectures. This allows the comparison of the efficiency of different network architectures independently of a task. Second, we introduce and experimentally validate a heuristic method to estimate the neural network capacity requirement for a given dataset and labeling. This allows an estimate of the required size of a neural network for a given problem. We conclude the article with a discussion on the consequences of sizing the network wrongly, which includes both increased computation effort for training as well as reduced generalization capability. 
\end{abstract}

\section{Introduction}
Most approaches to machine learning experiments currently involve tedious hyperparameter tuning. As the use of machine learning methods becomes increasingly important in industrial and engineering applications, there is a growing demand for engineering laws similar to the ones existing for electronic circuit design. Today, circuits can be drawn on a piece of paper and their behavior can be predicted exclusively based of engineering laws. Fully predicting the behavior of machine learning, as opposed to relying on trial and error, requires insights into the training and testing data, the available hypothesis space of a chosen algorithm, the convergence and other properties of the optimization algorithm, and the effect of generalization and loss terms in the optimization problem formulation. As a result, we may never reach circuit-level predictability. One of the core questions that machine learning theory focuses on is the complexity of the hypothesis space and what functions can be modeled, especially in connection with real-world data. Practically speaking, the memory and computation requirements for a given learning tasks are very hard to budget. This is especially a problem for very large scale experiments, such as on multimedia or molecular dynamics data. 

Even though artificial neural networks have been popular for decades, the understanding of the processes underlying them is usually based solely on anecdotal evidence in a particular application domain or task (see for example~\cite{morgan2012}). This article presents general methods to both measure and also analytically predict the experimental design for neural networks based on the underlying assumption that memorization is worst-case generalization.

We present 4 engineering rules to determine the maximum capacity of contemporary neural networks: 
\begin{enumerate}
\item The output of a single perceptron yields maximally one bit of information.
\item The capacity of a single perceptron is the number of its parameters (weights and bias) in bits.
\item The total capacity $C_{tot}$ of $M$ perceptrons in parallel is $C_{\text{tot}} = \sum_{i=1}^M C_i$ where $C_i$ is the capacity of each neuron.
\item For perceptrons in series (e.g., in subsequent layers), the capacity of a subsequent layer cannot be larger than the output of the previous layer. 
\end{enumerate}

After presenting related work in Section~\ref{sec:prior} and summarizing MacKay's proof in Section~\ref{sec:mackay}, we derive the above principles in Sections~\ref{sec:infomodel} and \ref{sec:combining}. In Section~\ref{sec:heuristic}, we then present and evaluate a heuristic approach for fast estimation of the required neural network capacity, in bits, for a given training data set. The heuristic method assumes a network with static, identical weights. Even if such a network will still be able to approximately learn any labeling for a data set, it would require too many neurons. We then assume that training is able to cut down the number of parameters exponentially, when compared to the untrained network. Section~\ref{sec:generalization} discusses the practical implications of memory capacity for generalization. Finally, in Section~\ref{sec:conclusion}, we conclude the article with future work directions.  

\begin{figure}[tb!]
\centering
\includegraphics[width=0.48\textwidth]{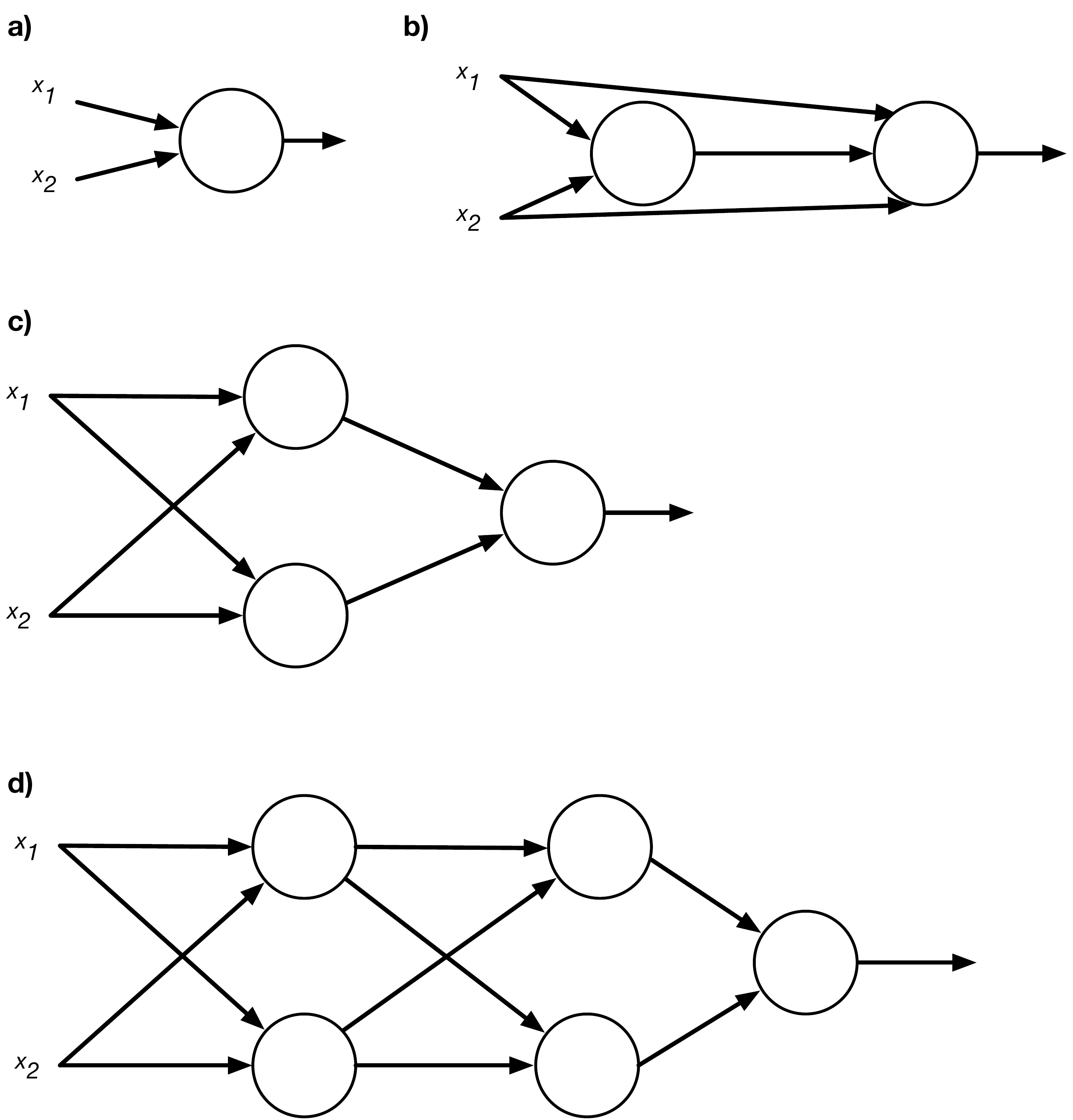}
\caption{\label{f:examples}The perceptron a) has $3$ bits of capacity and can therefore memorize the $14$ Boolean functions of two variables that can have a truth table of $8$ or less states (removing redundant states). The shortcut network b) has $3+4=7$ bits of capacity and can therefore implement all $16$ Boolean functions. The 3-layer network in c) has $6+\min(3,2)=8$ bits of capacity. Last but not least the deep network d) has $6+\min(6,2)+\min(3,2)=10$ bits of capacity.} 
\end{figure}

\section{Related Work}
\label{sec:prior}
\begin{figure*}
\centering
\fbox{\includegraphics[width=0.8\textwidth]{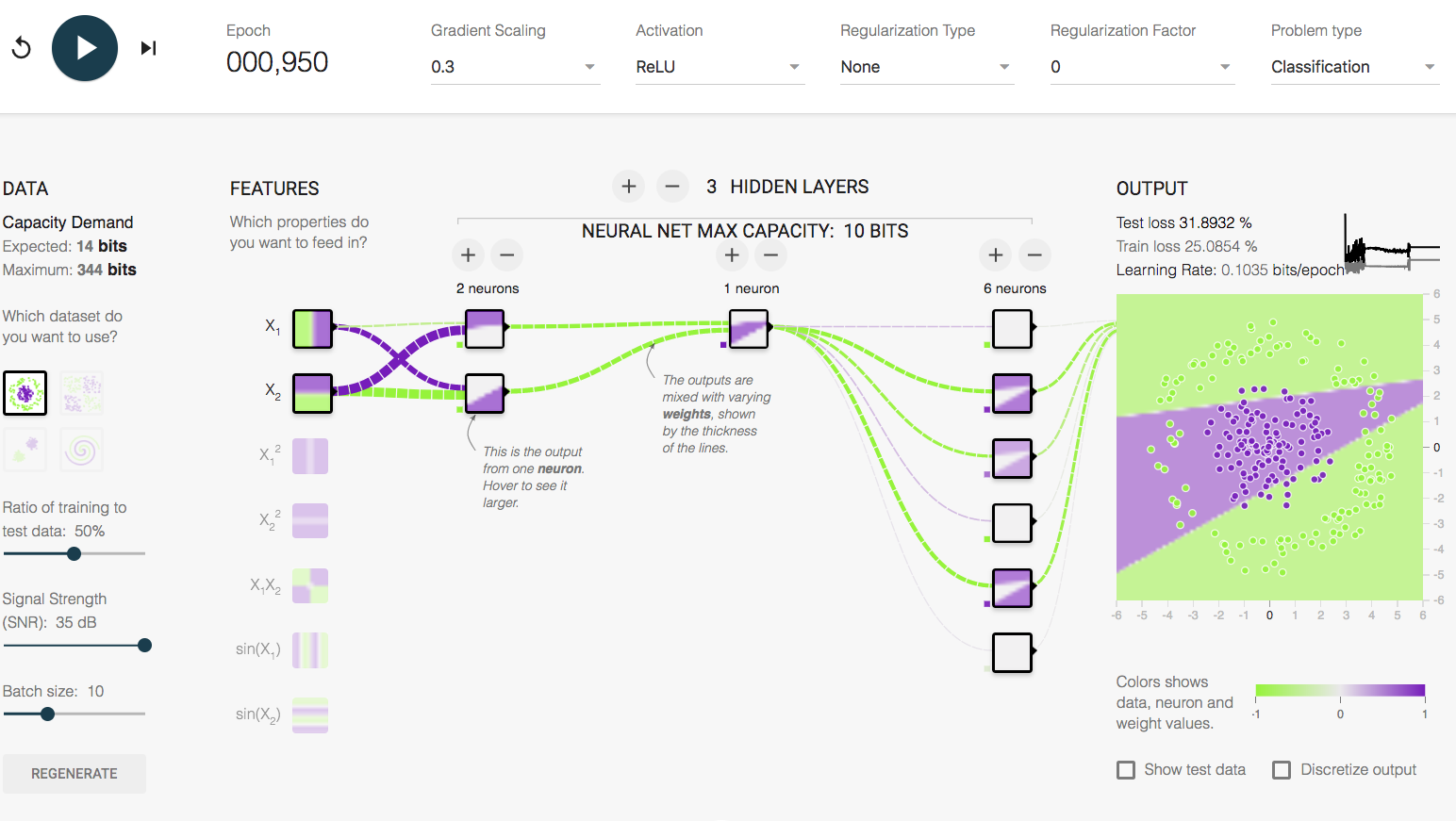}}
\caption{\label{f:dpitfmeter} Our web demo based on the Tensorflow Playground (link see Section~\protect\ref{sec:conclusion}) showing the Principle~\protect\ref{def:prin4} in action. The third hidden layer is dependent on the second hidden layer. Therefore it only holds $1$ bit of information (smoothed by the activation function) despite consisting of 6 neurons with a stand-alone capacity of $18$ bits.}
\end{figure*}

The perceptron was introduced in 1958~\cite{Rosenblatt1958} and since then, it has been extended in many variants, including, but not limited to, the structures described in~\cite{Crammer2006,Dekel2008,KrellPhd2015,KrellOc2015}. The perceptron uses a $k$-dimensional input and generates the output by applying a linear function to the input, followed by a gating function. The gating function is typically the identity function, the sign function, a sigmoid function, or the rectified linear unit (ReLU)~\cite{He2015,Nair2010}. Motivated by brain research~\cite{Feldman1982}, perceptrons are stacked together to networks and they are usually trained by a chain rule known as backpropagation~\cite{Rumelhart1986,Rumelhart1988}. 

Even though perceptrons have been utilized for a long time, its capacities have been rarely explored beyond discussion of linear separability. Moreover, catastrophic forgetting has so far not been explained satisfactorily. Catastrophic forgetting~\cite{McCloskey1989,Ratcliff1990} is a phenomenon consisting in the very quick loss of the network's capability to classify the first set of labels, when the net is first trained on one set of labels and then on another set of labels. Our interpretation of the cause of this phenomenon is that it is simply a capacity overflow.

One of the largest contributions to machine learning theory comes from Vapnik and Chervonenkis~\cite{Vapnik2000}, including the Vapnik-Chervonenkis (VC) dimension. The VC dimension has been well known for decades~\cite{Vapnik1971} and is defined as the largest natural number of samples in a dataset that can be shattered by a hypothesis space. This means that for a hypothesis space having VC dimension $D_{VC}$, there exists a dataset with $D_{VC}$ samples such that for any binary labeling ($2^{D_{VC}}$ possibilities) there exists a perfect classifier $f$ in the hypothesis space, that is, $f$ maps the samples perfectly to the labels. Due to perfect memorizing, it holds that $D_{VC}=\infty$ for 1-nearest neighbor. Tight bounds have so far been computed for linear classifiers ($k+1$) as well as decision trees~\cite{Asian2009}. The definition of VC dimension comes with two major drawbacks, however. First, it only considers the potential hypothesis space but not other aspects, like the optimization algorithm, or loss and regularization function affecting the choice of the hypothesis~\cite{Arpit2017}. Second, it is sufficient to provide only one example of a dataset to match the VC dimension. Hence, given a more complex structure of the hypothesis space, the chosen data can take advantage of this structure. As a result, shatterability can be increased by increasing the structure of the data. While these aspects do not matter much for simple algorithms, they constitute a major point of concern for deep neural networks. In~\cite{Vapnik1994}, Vapnik et al. suggest to determine the VC dimension empirically, but state in their conclusion that the described approach does not apply to neural networks as they are ``beyond theory". So far, the VC dimension has only been approximated for neural networks. For example, Mostafa argued loosely that the capacity must be bounded by $N^2$ with $N$ being the number of perceptrons~\cite{Mostafa1989}. Recently,~\cite{Shwartz2014} determined in their book that for a sigmoid activation function and a limited amount of bits for the weights, the loose upper bound of the VC dimension is $\mathcal{O}(|E|)$ where $E$ is the set of edges and consequently $|E|$ the number of non-zero weights. Extensions of the boundaries have been derived for example for recurrent neural networks~\cite{Koiran1998} and networks with piecewise polynomials~\cite{bartlett1999almost} and piecewise linear~\cite{Harvey2017} gating functions. Another article~\cite{Koiran1997} describes a quadratic VC dimension for a very special case. The authors use a regular grid of $n$ times $n$ points in the two dimensional space and tailor their multilayer perceptron directly to this structure to use only $3n$ gates and $8n$ weights.

One measure that handles the properties of given data is the Rademacher complexity~\cite{Bartlett2001}. For understanding the properties of large neural networks, Zhang et al.~\cite{zhang2017} recently performed randomization tests. They show that their observed networks can memorize the data as well as the noise. This is proven by evaluating that their neural networks perfectly learn with random labels or with random data. This shows that the VC dimension of the analyzed networks is above the size of the used dataset. But it is not clear what the full capacity of the networks is. This observation also explains why smaller size networks can outperform larger networks. A more elaborate extension of this evaluation has been provided by Arpit et al.~\cite{Arpit2017}. 

\begin{figure*}
\centering
\includegraphics[width=0.49\textwidth]{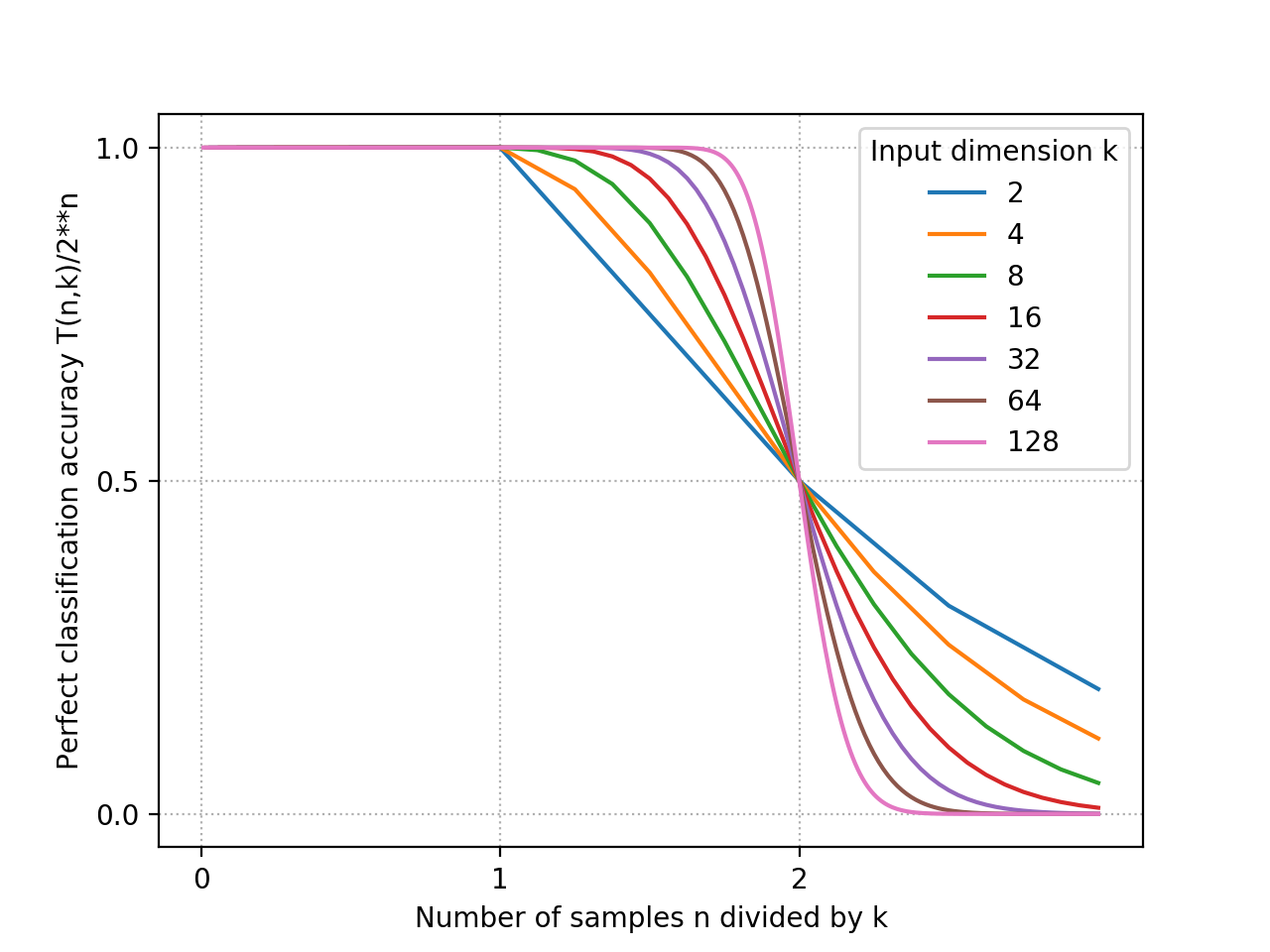}
\includegraphics[width=0.49\textwidth]{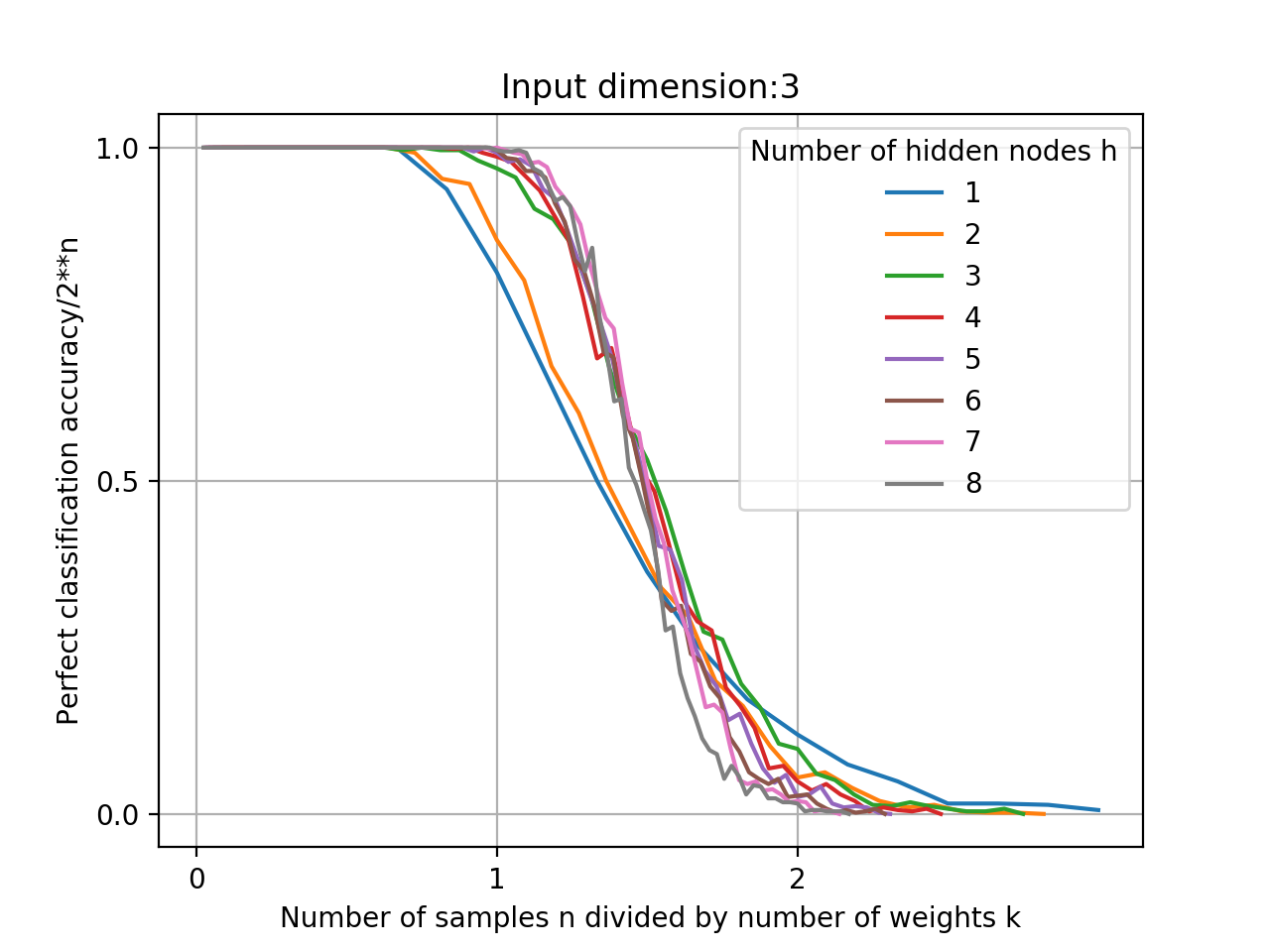}
\caption{
Left: \label{f:tnk} Characteristic curve examples of the $T(n,k)$ function for different input dimensions $k$ and the two crucial points at $n=k$ for the VC dimension and $n=2k$ for the MacKay capacity. Right: \label{f:realtnk} Measured characteristic curve example for different number of hidden layers for a configuration of scikit-learn\protect\cite{friedland2018caplaw}. The tools to measure and compare the characteristic curves of concrete neural network implementations are available in our public repository (see Section~\protect\ref{sec:conclusion}). 
}
\end{figure*}

Summarizing the contribution by~\cite{cover1965}, MacKay is the first to interpret a perceptron as an encoder in a Shannon communication model (\cite{mackay2003}, Chapter~40). MacKay's use of the Shannon model allows the measurement of the memory capacity of the perceptron in bits. Furthermore, it allows for the discussion of a perceptron's capabilities, without taking into account the number of bits used to store the weights (64\,bit doubles, real-valued, etc.). He also points out that there are two distinct transition points in the error measurement. The first one is discontinuous and happens at the VC dimension. For a single perceptron with offset, that point is $D_{VC}=k+1$, when $k$ is the dimensionality of the data. Below this point the error should be $0$, given perfect training, because the perceptron is able to generate all possible shatterings of the hypothesis space. For clarification, we summarize this proof in Section~\ref{sec:mackay} and present initial work on an extension in~\cite{friedland2018caplaw}.

Another important contribution using information theory comes from Tishby~\cite{Tishby2015}. They use the information bottleneck principle to analyze deep learning. For each layer, the previous layers are treated as an encoder that compresses the data $X$ to some better representation $T$ which is then decoded to the labels $Y$ by the consecutive layers. By calculating the respective mutual information $I(X,T)$ and $I(T,Y)$ for each layer they analyze networks and their behavior during training or when changing the amount of training data. Our Principle~\ref{def:prin4} is a direct consequence of his work. 

This questions of generalization and network architecture have recently become a heated academic discussion again as deep learning surprisingly seems to outperform shallow learning. For deep learning, single perceptrons with a nonlinear and continuous gating function are concatenated in a layered fashion. Techniques like convolutional filters, drop out, early stopping, regularization, etc., are used to tune performance, leading to a variety of claims about the capabilities and limits of each of these algorithms (for example~\cite{zhang2017}). We are aware of recent questioning of the approach of discussing the memory capacity of neural networks~\cite{Arpit2017,zhang2017}. Occam's razor~\cite{blumer1987} dictates that one should follow the path of least assumptions, as perceptrons were initially conceived as a "generalizing memory", as detailed for example, in the early works of  Widrow~\cite{widrow1962}. This approach has also been suggested by~\cite{Mostafa1989} and, as mentioned earlier, later explained in depth by MacKay~\cite{mackay2003}. Also, the Ising model of ferromagnetism, which is a well-known model used to explain memory storage, has already been reported to have similarities to perceptrons~\cite{gardner1987,gardner1988} and also to the neurons in the retina~\cite{tkacik2006ising}. 

\section{Capacities of a Perceptron}
\label{sec:mackay}
Here we summarize the proof elaborated in~\cite{cover1965} and \cite{mackay2003}, Chapter~40. 

The functionality of a perceptron is typically explained by the XOR example (i.\,e., showing that a perceptron with $2$ input variables $k$, which can have $2^k=4$ states, can only model $14$ of the $2^{2^k}=16$ possible output functions). XOR and its negation cannot be linearly separated by a single threshold function of two variables and a bias. For an example of this explanation, see~\cite{Rojas1996}, section~3.2.2. Instead of computing binary functions of $k$ variables, MacKay effectively changes the computability question to a labeling question: given $n$ points in general position, how many of the $2^n$ possible labelings in $\{0,1\}^n$ can be trained into a perceptron. Just as done by~\cite{cover1965,Rojas1996}, MacKay uses the relationship between the input dimensionality of the data $k$ and the number of inputs $n$ to the perceptron, which is denoted by a function $T(n,k)$ that indicates the number of ``distinct threshold functions'' (separating hyperplanes) of $n$ points in general position in $k$ dimensions. The original function was derived by~\cite{schlaefli1852}. It can be calculated as:
\begin{equation}
\label{eq:tnk2}
T(n,k)=2\sum_{l=0}^{k-1}\genfrac(){0pt}{0}{n-1}{l}
\end{equation}

Most importantly, it holds that
\begin{equation}
T(n,k)=2^n \text{ } \forall k: k\geq n.
\end{equation}
This allows to derive the VC dimension $D$ of a neuron with $k$ parameters shattering a set of $n$ points in general position. The number of possible binary labelings for $n$ points is $2^n$ and $T(n,n=k)=2^n$. 
This is the $D=k$, since all possible labelings of the $k=n$ points can be realized. 

When $k<n$, the $T(n,k)$ function follows a calculation scheme based on the Pascal Triangle~\cite{coolidge1949story}, which means that the loss due to incomplete shattering is still predictable. MacKay uses an error function based on the cumulative distribution of the standard Gaussian to perform that prediction and approximate the resulting distribution. More importantly, he defines a second point, at which only $50\,\%$ of all possible labelings can be separated by the binary classifier. He proofs this point to be at $n=2k$ for large $k$ and illustrates that there is a sharp continuous drop in performance at this point. MacKay then follows Cover's conclusion that the information theoretic capacity of a perceptron is $2k$. We call this point MacKay dimension in~\cite{friedland2018caplaw}.

When comparing and visualizing the $T(n,k)$ function, it is only natural to normalize function values by the number of possible labelings $2^n$ and to normalize the argument by the number of inputs $k$ which is equal to the capacity of the perceptron. Figure~\ref{f:tnk} displays these normalized functions for different input dimensions $k$. The functions follow a clear pattern like the characteristic curves of circuit components in electrical engineering. 


\section{Information Theoretic Model}
\label{sec:infomodel}
\begin{figure*}
\centering
\includegraphics[width=0.7\textwidth]{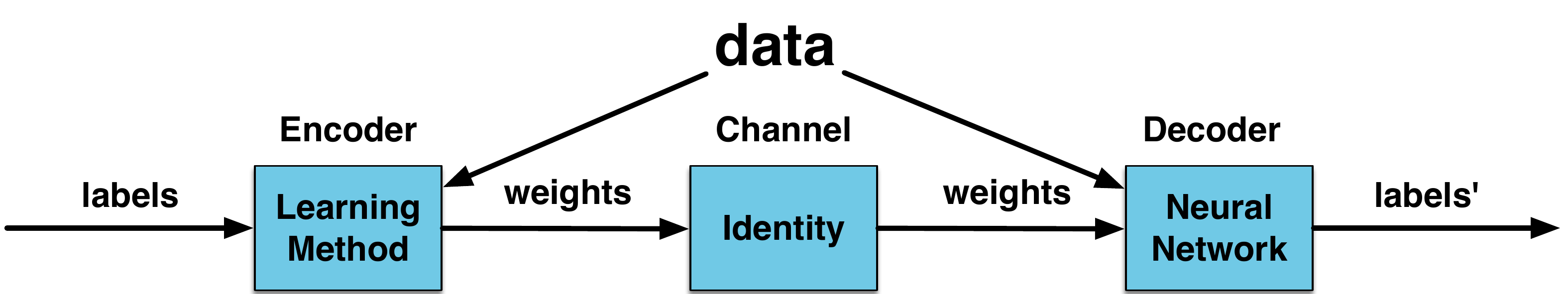}
\caption{\label{f:shannon}
Shannon's communication model applied to labeling in machine learning. A dataset consisting of $n$ sample points and the ground truth labeling of $n$ bits are sent to the neural network. The learning method converts it into a parameterization (i.\,e., network weights). In the decoding step, the network then uses the weights together with the dataset to try to reproduce the original labeling.}
\end{figure*}
To the best of our knowledge, MacKay is the first person to interpret a perceptron as an encoder in a Shannon communication model (\cite{mackay2003}, Chapter~40). In our article, we use a slightly modified version of the model depicted in Fig.~\ref{f:shannon}. 

As explained in Section~\ref{sec:mackay}, the input of the encoder are $n$ points in general position and a random labeling. The output of the encoder are the weights of a perceptron. The decoder receives the (perfectly learned) weights over a lossless channel. The question is: given the received set of weights and the knowledge of the data, can the decoder reconstruct the original labels of the points? In other words, the perceptron is interpreted as memory that stores a labeling of $n$ points relative to the data: how much information can then be stored by training a perceptron? We address this question by interpreting MacKay's definition of neuron capacity as a \textit{memory capacity}. 
The use of Shannon's model has an advantage: The mathematical framework of information theory can be applied to machine learning. Moreover, it allows to predict and measure neuron capacity in the unit of information: bits. 

We are interested in an upper bound. Therefore, we are only interested in the cases where we can guarantee lossless reproduction of the trained function; in other words, we are interested in the lossless memory capacity of neurons and networks of neurons. The definition of general position used in the previous section is typically used in linear algebra and is the most general case needed for a perceptron that uses a hyperplane for linear separation (see also Table~1 in~\cite{cover1965}). For neural networks, a stricter setting is required because they can implement arbitrary non-linear separations. We must therefore assume that the data points are in completely random positions. This is, the coordinates of the data points are equiprobable. 

\section{Networks of Perceptrons}
\label{sec:combining}

For the remainder of this article, we will assume that the network is a feedforward network consisting of traditional perceptrons (threshold units with activation function) with real-valued weights. Each unit has a bias, which counts as an additional real-valued weight~\cite{Rojas1996,mackay2003}. We will additionally assume that the perceptrons are part of a neural network embedded in the model depicted in Figure~\ref{f:shannon}, thus solving a binary labeling task. Because our discussion concerns the upper bounds, it is agnostic about training algorithms.

We define perceptrons to be \textit{in parallel} when they are connected to the same input. A \textit{layer} is a set of perceptrons in parallel. We define perceptrons to be \textit{in series} when they are connected in such a way that as the ones exclusively relying on the outputs of other perceptrons.

We note that Figure~\ref{f:examples} b) shows a perceptron that is connected in parallel. 

\begin{principle}{The output of a single perceptron yields maximally one bit of information}
\label{def:prin1}
\end{principle}

A perceptron uses a decision function $f(\vec{w},\vec{x},b)$ of the form\\
\begin{align}
f(\vec{w},\vec{x},b) = {\begin{cases}1&{\text{if }}\ \vec{w}\cdot \vec{x}>b\\0&{\text{otherwise}}\end{cases}} 
\end{align}
where $\vec{x} = \{ x_1, x_2, \ldots, x_N \}$ and $\vec{w} = \{ w_1, w_2, \ldots, w_N \}$ are real vectors and $b$ is a real scalar. Therefore, $\vec{w} \cdot \vec{x}$ represents a dot product:
\begin{align}
    \vec{w}\cdot \vec{x} = \sum _{i=1}^{N}w_{i}x_{i} \label{eqn:scalprod}
\end{align}
Because the inequality describes a binary condition (it is either greater or not), it follows that each perceptron ultimately behaves as a binary classifier, thus outputting a symbol $o=f(\vec{w},\vec{x},b) \in \{0,1\}$. If each state of $o$ is equiprobable, the information content encoded in the output of the perceptron is $\log_2(2)=1$ bit, else, if each state of $o$ is not equiprobable, the information content it is less than $1$ bit. It is worth remarking that an analytic approximation of the step function $f(\vec{w},\vec{x},b)$, for example a sigmoid, a rectified linear unit, or any other space dividing function, does not affect the aforementioned analysis. This is guaranteed by the data processing inequality~\cite{mackay2003} (p. 144).

\begin{principle}{The lossless storage capacity of a single perceptron is the number of parameters in bits.}
\label{def:prin2}
\end{principle}

This follows intuitively from Section~\ref{sec:mackay}, because $n$ bits of labels can be stored with $k=n$ parameters. This is, each parameter models one bit of labeling. However, confusion often arises over the fact that the weights are assumed real-valued. We therefore introduce the following lemma showing that a perceptron behaves analogous to a memory cell. This is, given fixed random input, it can model $2^k$ different output states, where $k$ is the number of parameters stored by the perceptron.  

Assume a perceptron as defined in Principle~\ref{def:prin1} in the model defined in Section~\ref{sec:infomodel}. Let $C(k)$ be the number of bits of labeling storable by $k$ parameters. 

\begin{lemma}[Lossless Storage Capacity of a Perceptron]
\label{lem:lmdperc}
$ C(k) = k$
\end{lemma}
\begin{proof}
Let us consider a case distinction over $b$.

Case 1: $b=0$\\
We now rewrite Eq.~\ref{eqn:scalprod} as:
\begin{align}
    \sum _{i=1}^{N}s_{i}|w_{i}|x_{i}
\end{align}
where $|w_{i}|$ is the absolute value of $w_i$ and $s_i$ is the sign of each $w_i$, this is $s_i \in \{-1,1\}$. 

It is now clear that, given an input $x_i$, the choice of $s_i$ in training is the only determining factor for the outcome of $f(\vec{w},\vec{x},b)$. The values of $|w_i|$ merely serve as scaling factors\footnote{Such scaling maybe important for generalization and training but is not relevant for computing the decision changing capabilities.}.      

Since $s_i \in \{-1,1\}$ and $|\{-1,1\}|=2$, it follows that each $s_i$ can be encoded using $\log_2(2)=1$ bit. This is, the maximum number of encodable outcome changes for $f(\vec{w},\vec{x},b)$ is $N$. This inevitably results in the memory capacity of a perceptron being $C(N)=N$.

Case 2: $b\neq 0$\\
Using the same approach as above, we begin by separating the bias and its sign: $b=s_b|b|$, where $|b|$ is the absolute value of $b$ and $s_b$ is the sign of $b$, this is $s_b \in \{-1,1\}$. We can now reformulate the equation as:
\begin{align}
    \sum _{i=1}^{N}s_{i}|w_{i}|x_{i} &> |b|s_b \\
    \frac{1}{s_b}\sum _{i=1}^{N}s_{i}|w_{i}|x_{i} &> |b|
\end{align}

Since $s_b$ is not dependent on $i$, $s_b$ can only be trained to correct all decisions at once. $|b|$ is strictly positive. This is, the inequality can be decided just by comparing the sign of $s_b$ and the sign of the sum. Again, $s_b \in \{-1,1\}$ and thus $s_b$ encodes $\log_2(2)=1$ bit.  As a result, $b$ contributes $1$ bit of memory capacity. In total, a perceptron with non-zero bias can therefore maximally memorize $N+1$ bits of changes to the outcomes of $f(\vec{w},\vec{x},b)$. Since $k=N+1$, it inevitably follows that $C(k)=k$.
\end{proof}

\begin{principle}{The total capacity $C_{tot}$ of $M$ perceptrons in parallel is:
\begin{align}
    C_{\text{tot}} = \sum_{i=1}^M C_i
\end{align}
where $C_i$ is the capacity of each neuron.
\label{def:prin3}}
\end{principle}

Consistent with MacKay's interpretation, connecting, for example, two perceptrons  in parallel is analogous to using two memory cells with capacity $C_1$ and $C_2$. The storage capacity of such a circuit is maximally $C_{tot}=C_1+C_2$ bits.

For the following lemma we assume two perceptrons connected to the same input, each with a number of parameters $k_1$ and $k_2$. Due to the associativity of addition, We can do this Without loss of generality.

\begin{lemma}[Perceptrons in parallel]
\label{the:lmpar}
$C(k_1+k_2)=k_1+k_2$
\end{lemma}
\begin{proof}
We know from Lemma~\ref{lem:lmdperc} that $C(k_1)=k_1$ and $C(k_2)=k_2$. Since we assume all points of the data to be in equiprobable positions, each perceptron $i$ can now maximally label $k_i$ points independently. This is, the two perceptrons can maximally label $k_1+k_2$ points. This is, $C(k_1+k_2)=k_1+k_2$.
\end{proof}

\begin{principle}{For perceptrons in series, the capacity of a subsequent layer cannot be larger than the largest possible amount of information output of the previous layer.}
\label{def:prin4}
\end{principle}

As explained in Section~\ref{sec:prior}, Tishby~\cite{Tishby2015} treats each layer in a deep perceptron network as an encoder for a subsequent layer. The work analyzes the mutual information between layers and points out that the data processing inequality holds between them, both theoretically and empirically. We are able to confirm this result and note that channel capacity $C$ in general is defined as $C=\sup _{{p_{X}(x)}}I(X;Y)$, where the supremum is taken over all possible choices of $p_{X}(x)$(~\cite{shannon1948bell}). The data processing inequality (\cite{mackay2003}, p. 144) states that if $X\rightarrow Y\rightarrow Z$ is a Markov chain then $I(x;y)\geqslant I(x;z)$, where $I(x;y)$ is the mutual information. In our model (See Section~\ref{sec:infomodel}), the channel is the identity channel and the label distribution is assumed as equiprobable. These two assumptions make the channel capacity identical to the memory capacity and to the mutual information. As a result, the capacity of a subsequent layer is upper bounded by the output of the previous layer. 

Without loss of generality, we assume two layers of perceptrons. The output of perceptron layer 1 is the sole input for perceptron layer 2. We denote the total capacity of layer 2 with $C_{L2}$, the number of parameters in layer 2 with $k_{L2}$ and the number of bits in the output of layer 1 with $o_{L1}$. 

\begin{lemma}[Perceptrons in Series]
\label{the:lmser}
$C_{L2} = \min(C(k_{L2}),o_{L1})$
\end{lemma}
\begin{proof}
Let create the Markov chain $X\rightarrow Y \rightarrow Z$, where $X$ is the random variable representing the input to layer 1, $Y$ is the random variable representing the output of layer 1 and $Z$ is representing the output of layer 2. It is clear that the $\sup_{{p_{X}(x)}}I(Y;Z)$ is bounded by $I(X;Y)$, which we know to be $o_{L1}$. If $o_{L1}>C(k_{L2})$, then $C(K_{L2})$ limits the number of bits that can be stored in layer 2. If $o_{L1} \leq C(k_{L2})$, then the data processing inequality does not allow for the creation of information and $\sup _{{p_{X}(x)}}I(Y;Z) \leq o_{L1}$. As a consequence, $C_{L2} = \min(C(k_{L2}),o_{L1})$
\end{proof}

When generalizing to more than two layers, it is important to keep in mind that any capacity constraint from an earlier layer will upper bound all subsequent layers. This is, capacity can never increase in subsequent layers. Note that the input layer counts as a layer as well. Figure~\ref{f:dpitfmeter} shows a screen shot of our neural network capacity web demo (link see Section~\ref{sec:conclusion}) with an example of Principle~\ref{def:prin4} in action. Figure~\ref{f:examples} discusses various architecture capacities practically applying the computation principles presented here. 

There is a notable illusion that sometimes makes it seem that Principle~\ref{def:prin4} does not hold. In training, weights are initialized, for example at random. This initial configuration can create the illusion that a layer has more states available than dictated by the principle. For example, a layer that has only $1$ bit of capacity using Principle~\ref{def:prin4} can be in more than $2$ states before the weights have been updated in training based on the information passed by a previous layer.

\subsection{Measuring Capacity}
It is possible to practically measure the capacity of concrete neural networks implementations with varying architectures and learning strategies. This is done by generating $n$ random data data points in $d$ dimensions and training the network to memorize all possible $2^n$ binary labeling vectors. Once a network is not able to learn all labelings anymore, we reached capacity. While this effectiveness measurement is exponential in run time, it only needs to be performed on a small representative subnet as capacity scales linearly.

We found that the effectiveness of neural network implementations actually varies dramatically (always below the theoretical upper limit).  Therefore capacity measurement alone allows for a task-independent comparison of neural network variations.  Our experiments show that linear scaling holds practically and our theoretical bounds are actionable upper bounds for engineering purposes. All the tested threshold-like activation functions, including sigmoid and ReLU exhibited the predicted behavior -- just as explained in theory by the data processing inequality. Our experimental methodology serves as a benchmarking tool for the evaluation of neural network implementations. Using points in random position, one can test any learning algorithm and network architecture against the theoretical limit both for performance and efficiency (convergence rate). Figure~\ref{f:tnk} (right) shows an example measurement curve. These results as well as all tools are available in our public repository (See Section~\ref{sec:conclusion}).

\section{Capacity Requirement Estimate}
\label{sec:heuristic}
\begin{algorithm}
\begin{algorithmic}
\Require $data$: array of length $i$ contains $d$-dimensional vectors $x$, $labels$: a column of $0$ or $1$ with length $i$ 
\Procedure{$MaxCapReq$}{$(data, labels)$}
\State $thresholds \gets 0$
\ForAll{$i$}
\State $table[i] \gets (\sum{x[i][d]}, label[i])$
\State $sortedtable \gets sort(table, key=column~0)$
\State $class \gets 0$
\EndFor
\ForAll{$i$}
\If{not $sortedtable[i][1] == class $}
\State $class \gets sortedtable[i][1]$
\State $thresholds \gets thresholds+1$
\EndIf
\EndFor
\State $maxcapreq \gets thresholds*d+thresholds+1$
\State $expcapreq \gets \log_2(thresholds+1)*d$
\State {\bf print} "Max:~"$+maxcapreq+$"~bits"
\State {\bf print} "Exp:~"$+expcapreq+$"~bits"
\EndProcedure
\caption{\label{alg:maxnn}Calculating the maximum and approximated expected capacity requirement of a binary classifier neural network for given training data.}
\end{algorithmic}
\end{algorithm}

The upper-bound estimation of the capacity allows the comparison of the efficiency of different architectures independent of a task. However, sizing a network properly to a task requires an estimate of the required capacity. We propose a heuristic method to estimate the neural network capacity requirement for a given dataset and labeling. 

The exact memorization capacity requirement based on our model in Figure~\ref{f:shannon} is the minimum description length of the data/labels table that needs to be memorized. In practice, this value is almost never given. Furthermore, in a neural network, the table is recoded using weighted dot-product threshold functions, which, as discussed in Section~\ref{sec:mackay}, has intrinsic compression capabilities. This is, often the labels of $n$ points can be stored with less than $n$ parameters. As we have done throughout the article, we will ignore the compression capabilities of neurons and work with the worst case.

\subsection{Upper Limit Network Size}
This section presents our proposed heuristic for a worst case sized network. Our idea for the heuristic method stems from the definition of the perceptron threshold function (see Principle~\ref{def:prin1}). We observe that the dot product has $d+1$ variables that need to be tuned, with $d$ being the dimensionality of the input vector $x$. This makes perceptron learning and backpropagation NP-complete~\cite{blum1989training}. However, for an upper limit estimation, we chose to ignore the training of the weights $w_i$ by fixing them to $1$: we only train the biases. This is done by calculating the dot products with $w_i:=1$, essentially summing up the data rows of the table. The result is a two-column table with these sums and the classes. We now sort this two-column table by the sums before we iterate through it and record the need of a threshold every time a class change occurs. Note that we can safely ignore column sums with the same value (collisions): If they don't belong to the same class, they count as a threshold. If an actual network was built, training of the weights would potentially resolve this collision. As a last step, we take the number of thresholds needed and estimate the capacity requirement for a network by assuming that each threshold is implemented by a neuron in the hidden layer firing $0$ or a $1$. The number of inputs for these neurons is given by the dimensionality of the data. We then need to connect a neuron in the output layer that has the number of hidden layer neurons as input. The threshold of that output neuron is $0$ and the input weights are $+1$ for class $1$ and $-1$ for class $0$. The reader is encouraged to check that such a network is able to label any table (ignoring collisions). Our algorithm is bounded by the runtime of the sorting, which is $\mathcal{O}(n \log(n))$ in the best case. Since we are able to effectively create a network that memorizes the labeling given the data without tuning the weights, we consider this the upper limit network. Any network that uses more parameters would therefore be wasting resources. Figure~\ref{alg:maxnn} shows pseudo code for this algorithm and the expected capacity presented in the next section. 

\begin{table*}
\centering
\resizebox{\textwidth}{!}{%

\begin{tabular}{|l|l|l|l|}
\hline
\rowcolor[HTML]{C0C0C0} 
\textbf{Dataset} & \textbf{Max Capacity Requirement} & \textbf{Expected Capacity Requirement} & \textbf{Validation (\% accuracy)} \\ \hline
AND, 2 variables & 4 bits & 2 bits & 2 bits (100\%) \\ \hline
XOR, 2 variables & 8 bits & 4 bits & 7 bits (100\%) \\ \hline
\begin{tabular}[c]{@{}l@{}}Separated Gaussians (100 samples)\end{tabular} & 4 bits & 2 bits & 3 bits (100\%) \\ \hline
2 Circles (100 samples) & 224 bits & 12 bits & 12 bits (100\%) \\ \hline
Checker pattern (100 samples) & 144 bits & 12 bits & 12 bits (100\%) \\ \hline
Spiral pattern (100 samples) & 324 bits & 14 bits & 24 bits (98\%) \\ \hline
ImageNet: 2000 images in 2 classes & 906984 bits & 10240 bits & 10253 bits (98.2 \%) \\ \hline
\end{tabular}
}
\caption{\label{t:experiments} Experimental validation of the heuristic capacity estimation method using the structures available both in our public repository and in the online demo.}
\end{table*}

\subsection{Approximately Expected Capacity}
We estimate the expected capacity by assuming that training the weights and biases is maximally effective. This is, it can cut down the number of threshold comparisons exponentially to $\log_2(n)$ where $n$ is the number of thresholds. The rationale for this choice is that a neural network effectively takes an input as a binary number and matches it against stored numbers in the network to determine the closest match. The output layer then determines the class for that match. That matching is effectively a search algorithm which in the best case can be implemented in logarithmic time. We call this the approximately expected capacity requirement as we need to take into account that real data is never random. Therefore, the network might be able to compress by a factor of $2$ or even a much higher margin. 

\subsection{Experimental Results}
Table~\ref{t:experiments} shows experimental results for various data sets. We show the maximum and the approximately expected capacity as generated by the heuristic method. We then show the achieved accuracy using an actual validation experiment using a neural network of the indicated capacity. The AND classifier requires one perceptron without bias. We implemented XOR using a shortcut network (see also~\cite{Rojas1996}). The Gaussians and the circle, checker, and spiral patterns are available as part of the Tensorflow Playground. For the ImageNet experiment, we took 2000 random images from 2 classes (``hummingbird" and ``snow leopard") and in lieu of a convolution layer, we compressed all images aggressively with JPEG quality 20~\cite{friedland2018helmholtz}. The image channels were combined from RGB into only the Y component (grayscale). We then trained a 3-layer neural network and increased the capacity successively. The best result was achieved at the capacity shown in the table; fewer parameters made the memorization result worse -- all other parameters being the same (e.g. 94.6\% accuracy at 5\,kbit capacity, 97.3\% accuracy at 9\,kbit capacity and 97.9\% accuracy at 11\,kbit capacity). We note that image experiments like these are often anecdotal as many factors play into the actual achieved accuracy, including initial conditions of the initialization, learning rate, regularization, and others. We therefore made the scripts and data available for repetition in our public repository (Link see Section~\ref{sec:conclusion}). The results show that our approximation of the expected capacity is very close to the actual capacity. 

\section{From Memorization to\\ Generalization}
\label{sec:generalization}
Training the network with random points makes the upper bound neural network size analytically accessible because no inference (generalization) is possible and the best possible thing any machine learner can do is to memorize. This methodology, which is not restricted to neural networks, therefore operates at the lower limit of generalization. 

In reality, especially with a large set of samples, one is very unlikely to encounter data with equiprobable distribution. A network trained based on the principles presented here is therefore overfitting. A first consequence is that using more capacity than required for memorization wastes memory and computation resources. Secondly, it will complicate any attempt at explaining the inferences made by the network.

To avoid overfitting and to have a better chance of explaining the data in a human comprehensible way, it is therefore advisable to reduce the number of parameters. This is, again, consistent with Occam's razor. For a given task, we therefore recommended to size the neural network for memorization at first and then successively re-train the network while reducing the number of parameters. It is expected that accuracy on the training data reduces with the network capacity reduction. Generalization capability, which should be quantified by measuring accuracy against a validation set, should increase, however. In the best case, the network loses the ability to memorize the lowest significant digits of the training data. The lowest significant digits are likely insignificant with regard to the target function. This is, they are noise. Cutting the lowest-significant digits first, we expect the decay of training accuracy to follow a logarithmic curve (this was also observed in~\cite{friedland2018helmholtz}). Ultimately, the network with the smallest capacity that is still able to represent the data is the one that maximizes generalization and the chances at explainability. The best possible scenario is a single neuron that can represent an infinite amount of points (above and below the threshold).

\section{Conclusion and Future Work}
\label{sec:conclusion}
We present an alternative understanding of neural networks using information theory. The main trick, that is not specific to neural networks, is to train the network with random points. This way, no inference (generalization) is possible and the best thing any machine learner can do is to memorize. We then present engineering principles to quantify the capabilities of a neural network given it's size and architecture as memory capacity. This allows the comparison of the efficiency of different architectures independently of a task. Second, we introduce and experimentally validate a heuristic method to estimate the neural network capacity requirement for a given dataset and labeling. We then relate this result to generalization and outline a process for reducing parameters. The overall result is a method to better predict and measure the capabilities of neural networks.

Future work in continuation of this research will explore non-binary classification, recursive architectures, and self-looping layers. Moreover, further research into investigating convolutional networks, fuzzy networks and RBF kernel networks would help put these types of architectures into a comparative perspective. We will also revisit neural network training given the knowledge that we have gained doing this research. During the backpropagation step, the data processing inequality is reversed. This is, only one bit of information is actually transmitted backwards through the layers.  All results and the tools for measuring capacity and estimating the required capacity are available in our public repository: \url{https://github.com/fractor/nntailoring}.
An interactive demo showing how capacity can be used is available at: \url{http://tfmeter.icsi.berkeley.edu}.  

\section*{Acknowledgements}
This work was performed under the auspices of the U.S. Department of Energy by Lawrence Livermore National Laboratory under Contract DE-AC52-07NA27344. It was also partially supported by a Lawrence Livermore Laboratory Directed Research \& Development grants (17-ERD-096 and 18-ERD-021). IM release number LLNL-TR-758456. Mario Michael Krell was supported by the Federal Ministry of Education and Research (BMBF, grant no. 01IM14006A) and by a fellowship within the FITweltweit program of the German Academic Exchange Service (DAAD). This research was partially supported by the U.S. National Science Foundation (NSF) grant CNS 1514509. The views and conclusions contained in this document are those of the authors and should not be interpreted as representing the official policies, either expressed or implied, of any sponsoring institution, the U.S. government or any other entity. We want to cordially thank Ra\'ul Rojas for in depth discussion on the chaining of the $T()$ function. We also want to thank Jerome Feldman for discussions on the cognitive backgrounds, especially the concept of actionability. Kannan Ramchandran's intuition of signal processing and information theory was invaluable.  We'd like thank Sascha Hornauer for his advise on the imagenet experiments as well as Alexander Fabisch, Jan Hendrik Metzen, Bhiksha Raj, Naftali Tishby, Jaeyoung Choi, Friedrich Sommer, Alyosha Efros, Andrew Feit, and Barry Chen for their insightful advise. Special thanks go to Viviana Rever\'on for proofreading.

\bibliographystyle{abbrv}
\bibliography{main}

\begin{thebibliography}{10}

\bibitem{Mostafa1989}
Y.~Abu-Mostafa.
\newblock {Information theory, complexity and neural networks}.
\newblock {\em IEEE Communications Magazine}, 27(11):25--28, November 1989.

\bibitem{Arpit2017}
D.~Arpit, S.~Jastrz{\c{e}}bski, N.~Ballas, D.~Krueger, E.~Bengio, M.~S. Kanwal,
  T.~Maharaj, A.~Fischer, A.~Courville, Y.~Bengio, and S.~Lacoste-Julien.
\newblock {A Closer Look at Memorization in Deep Networks}, jun 2017.

\bibitem{Asian2009}
O.~Asian, O.~T. Yildiz, and E.~Alpaydin.
\newblock {Calculating the VC-dimension of decision trees}.
\newblock In {\em 24th International Symposium on Computer and Information
  Sciences}, pages 193--198. IEEE, sep 2009.

\bibitem{bartlett1999almost}
P.~L. Bartlett, V.~Maiorov, and R.~Meir.
\newblock Almost linear vc dimension bounds for piecewise polynomial networks.
\newblock In {\em Advances in Neural Information Processing Systems}, pages
  190--196, 1999.

\bibitem{Bartlett2001}
P.~L. Bartlett and S.~Mendelson.
\newblock {Rademacher and Gaussian Complexities: Risk Bounds and Structural
  Results}.
\newblock {\em Journal of Machine Learning Research}, 3:463--482, 2001.

\bibitem{blum1989training}
A.~Blum and R.~L. Rivest.
\newblock Training a 3-node neural network is np-complete.
\newblock In {\em Advances in neural information processing systems}, pages
  494--501, 1989.

\bibitem{blumer1987}
A.~Blumer, A.~Ehrenfeucht, D.~Haussler, and M.~K. Warmuth.
\newblock Occam's razor.
\newblock {\em Information processing letters}, 24(6):377--380, 1987.

\bibitem{coolidge1949story}
J.~L. Coolidge.
\newblock The story of the binomial theorem.
\newblock {\em The American Mathematical Monthly}, 56(3):147--157, 1949.

\bibitem{cover1965}
T.~M. Cover.
\newblock Geometrical and statistical properties of systems of linear
  inequalities with applications in pattern recognition.
\newblock {\em IEEE transactions on electronic computers}, EC-14(3):326--334,
  1965.

\bibitem{Crammer2006}
K.~Crammer, O.~Dekel, J.~Keshet, S.~Shalev-Shwartz, and Y.~Singer.
\newblock {Online Passive-Aggressive Algorithms}.
\newblock {\em Journal of Machine Learning Research}, 7:551 -- 585, 2006.

\bibitem{Dekel2008}
O.~Dekel, S.~Shalev-Shwartz, and Y.~Singer.
\newblock {The Forgetron: A Kernel-Based Perceptron on a Budget}.
\newblock {\em SIAM Journal on Computing}, 37(5):1342--1372, jan 2008.

\bibitem{Feldman1982}
J.~A. Feldman.
\newblock {Dynamic connections in neural networks}.
\newblock {\em Biological Cybernetics}, 46(1):27--39, dec 1982.

\bibitem{friedland2018caplaw}
G.~Friedland and M.~Krell.
\newblock A capacity scaling law for artificial neural networks.
\newblock {\em arXiv preprint arXiv:1708.06019}, 2018.

\bibitem{friedland2018helmholtz}
G.~Friedland, J.~Wang, R.~Jia, and B.~Li.
\newblock The helmholtz method: Using perceptual compression to reduce machine
  learning complexity.
\newblock {\em arXiv preprint arXiv:1807.10569}, 2018.

\bibitem{gardner1987}
E.~Gardner.
\newblock Maximum storage capacity in neural networks.
\newblock {\em EPL (Europhysics Letters)}, 4(4):481, 1987.

\bibitem{gardner1988}
E.~Gardner.
\newblock The space of interactions in neural network models.
\newblock {\em Journal of physics A: Mathematical and general}, 21(1):257,
  1988.

\bibitem{Harvey2017}
N.~Harvey, C.~Liaw, and A.~Mehrabian.
\newblock {Nearly-tight VC-dimension bounds for piecewise linear neural
  networks}.
\newblock {\em Proceedings of Machine Learning Research: Conference on Learning
  Theory, 7-10 July 2017, Amsterdam, Netherlands}, 65:1064--1068, mar 2017.

\bibitem{He2015}
K.~He, X.~Zhang, S.~Ren, and J.~Sun.
\newblock {Delving Deep into Rectifiers: Surpassing Human-Level Performance on
  ImageNet Classification}.
\newblock In {\em 2015 IEEE International Conference on Computer Vision
  (ICCV)}, pages 1026--1034. IEEE, dec 2015.

\bibitem{Koiran1997}
P.~Koiran and E.~D. Sontag.
\newblock {Neural Networks with Quadratic VC Dimension}.
\newblock {\em Journal of Computer and System Sciences}, 54(1):190--198, feb
  1997.

\bibitem{Koiran1998}
P.~Koiran and E.~D. Sontag.
\newblock {Vapnik-Chervonenkis dimension of recurrent neural networks}.
\newblock {\em Discrete Applied Mathematics}, 86(1):63--79, aug 1998.

\bibitem{KrellPhd2015}
M.~M. Krell.
\newblock {\em {Generalizing, Decoding, and Optimizing Support Vector Machine
  Classification}}.
\newblock Phd thesis, University of Bremen, Bremen, 2015.

\bibitem{KrellOc2015}
M.~M. Krell and H.~W{\"{o}}hrle.
\newblock {New one-class classifiers based on the origin separation approach}.
\newblock {\em {Pattern Recognition Letters}}, 53:93--99, feb 2015.

\bibitem{mackay2003}
D.~J.~C. MacKay.
\newblock {\em Information Theory, Inference, and Learning Algorithms}.
\newblock Cambridge University Press, New York, NY, USA, 2003.

\bibitem{McCloskey1989}
M.~McCloskey and N.~J. Cohen.
\newblock {Catastrophic Interference in Connectionist Networks: The Sequential
  Learning Problem}.
\newblock {\em Psychology of Learning and Motivation}, 24:109--165, 1989.

\bibitem{morgan2012}
N.~Morgan.
\newblock {Deep and Wide: Multiple Layers in Automatic Speech Recognition}.
\newblock {\em IEEE Transactions on Audio, Speech, and Language Processing},
  20(1):7--13, Jan 2012.

\bibitem{Nair2010}
V.~Nair and G.~E. Hinton.
\newblock {Rectified Linear Units Improve Restricted Boltzmann Machines}.
\newblock In {\em Proceedings of the 27th International Conference on
  International Conference on Machine Learning}, ICML'10, pages 807--814, USA,
  2010. Omnipress.

\bibitem{Ratcliff1990}
R.~Ratcliff.
\newblock {Connectionist models of recognition memory: constraints imposed by
  learning and forgetting functions.}
\newblock {\em Psychological review}, 97(2):285--308, apr 1990.

\bibitem{Rojas1996}
R.~Rojas.
\newblock {\em {Neural networks: a systematic introduction}}.
\newblock Springer-Verlag, 1996.

\bibitem{Rosenblatt1958}
F.~Rosenblatt.
\newblock {The perceptron: a probabilistic model for information storage and
  organization in the brain}.
\newblock {\em Psychological review}, 65(6):386--408, November 1958.

\bibitem{Rumelhart1986}
D.~E. Rumelhart, G.~E. Hinton, and R.~J. Williams.
\newblock {Learning Internal Representations by Error Propagation}.
\newblock In D.~E. Rumelhart, J.~L. McClelland, and C.~PDP Research~Group,
  editors, {\em {Parallel Distributed Processing: Explorations in the
  Microstructure of Cognition, Vol. 1}}, pages 318--362. MIT Press, Cambridge,
  MA, USA, 1986.

\bibitem{Rumelhart1988}
D.~E. Rumelhart, G.~E. Hinton, and R.~J. Williams.
\newblock {Learning Representations by Back-propagating Errors}.
\newblock In J.~A. Anderson and E.~Rosenfeld, editors, {\em Neurocomputing:
  Foundations of Research}, pages 696--699. MIT Press, Cambridge, MA, USA,
  1988.

\bibitem{schlaefli1852}
L.~Schl\"afli.
\newblock {\em {Theorie der vielfachen Kontinuit\"at}}.
\newblock Birkh\"auser, 1852.

\bibitem{Shwartz2014}
S.~Shalev-Shwartz and S.~Ben-David.
\newblock {\em {Understanding Machine Learning: From Theory to Algorithms}}.
\newblock Cambridge University Press, New York, NY, USA, 2014.

\bibitem{shannon1948bell}
C.~E. Shannon.
\newblock {The Bell System Technical Journal}.
\newblock {\em A mathematical theory of communication}, 27:379--423, 1948.

\bibitem{Tishby2015}
N.~Tishby and N.~Zaslavsky.
\newblock {Deep learning and the information bottleneck principle}.
\newblock In {\em 2015 IEEE Information Theory Workshop (ITW)}, pages 1--5,
  April 2015.

\bibitem{tkacik2006ising}
G.~Tkacik, E.~Schneidman, I.~Berry, J.~Michael, and W.~Bialek.
\newblock {Ising models for networks of real neurons}.
\newblock {\em arXiv preprint q-bio/0611072}, 2006.

\bibitem{Vapnik2000}
V.~N. Vapnik.
\newblock {The nature of statistical learning theory}.
\newblock {\em Springer}, 2000.

\bibitem{Vapnik1971}
V.~N. Vapnik and A.~Y. Chervonenkis.
\newblock {On the Uniform Convergence of Relative Frequencies of Events to
  Their Probabilities}.
\newblock {\em {Theory of Probability {\&} Its Applications}}, 16(2):264--280,
  jan 1971.

\bibitem{Vapnik1994}
V.~N. Vapnik, E.~Levin, and Y.~L. Cun.
\newblock {Measuring the VC-Dimension of a Learning Machine}.
\newblock {\em Neural Computation}, 6(5):851--876, sep 1994.

\bibitem{widrow1962}
B.~Widrow.
\newblock Generalization and information storage in network of
  adaline'neurons'.
\newblock {\em Self-organizing systems-1962}, pages 435--462, 1962.

\bibitem{zhang2017}
C.~Zhang, S.~Bengio, M.~Hardt, B.~Recht, and O.~Vinyals.
\newblock Understanding deep learning requires rethinking generalization.
\newblock In {\em International Conference on Learning Representations (ICLR)},
  2017.

\end{thebibliography}

\end{document}